\documentclass{article}
\usepackage{spconf,amsmath,epsfig}

\usepackage{amssymb,amsthm,mathtools}
\usepackage{color}

\newcommand{\Signal}{\mathbf{X}}
\newcommand{\Coeff}{\mathbf{A}}
\newcommand{\Db}{\mathbf{D}}
\newcommand{\Dbp}{\mathbf{D}^\prime}
\newcommand{\Ub}{\mathbf{U}}

\newcommand{\signal}{\mathbf{x}}
\newcommand{\coeff}{\boldsymbol{\alpha}}

\newcommand{\Lcal}{\mathcal{L}}
\newcommand{\Ncal}{\mathcal{N}}

\newcommand{\EE}{\mathbb{E}}
\newcommand{\PP}{\mathbb{P}}
\newcommand{\R}[1]{\mathbb{R}^{#1}}

\newcommand{\Afrak}{\mathfrak{A}}
\newcommand{\Dcal}{\mathfrak{D}}

\newcommand{\Pfrak}{\mathfrak{P}} % has to be defined
\newcommand{\Sph}{\mathfrak{S}}

\newcommand{\vardev}{t}

\newtheorem{theorem}{Theorem}
\newtheorem{lemma}[theorem]{Lemma}
\newtheorem{proposition}[theorem]{Proposition}

\theoremstyle{definition}

\newtheorem{example}{Example}

\theoremstyle{remark}

\newcommand\Mark[1]{\textsuperscript{#1}}

\title{On The Sample Complexity of Sparse Dictionary Learning}
\name{M. Seibert\Mark{1}, M. Kleinsteuber\Mark{1}\thanks{This work was partially supported Cluster of Excellence CoTeSys funded by the German DFG.},
 R. Gribonval\Mark{2}\thanks{This work was partially supported by the EU FP7, SMALL project,
FET-Open grant number 225913, and by the European Research Council, PLEASE project (ERC-StG-2011-277906)},
R. Jenatton\Mark{3,4}, F. Bach\Mark{3,5}}
\address{\Mark{1}Department of Electrical Engineering and Information Technology, TU M\"unchen, Munich, Germany.\\
\Mark{2}PANAMA Project-Team (INRIA \& CNRS).\\
\Mark{3}SIERRA Project-Team (INRIA Paris).\\
\Mark{4}Centre de Math\'ematiques Appliqu\'ees - Ecole Polytechnique (CMAP).\\
\Mark{5}Laboratoire d'informatique de l'\'ecole normale sup\'erieure (LIENS).\\
\normalsize\{m.seibert,kleinsteuber\}@tum.de, gribonval@inria.fr, r.jenatton@criteo.com, francis.bach@ens.fr}
\expandafter\def\expandafter\normalsize\expandafter{%
    \normalsize
    \setlength\abovedisplayskip{6pt}
    \setlength\belowdisplayskip{6pt}
    \setlength\abovedisplayshortskip{6pt}
    \setlength\belowdisplayshortskip{6pt}
}

\begin{document}
%\ninept
%
\maketitle
\begin{abstract}
In the synthesis model signals are represented as a sparse combinations of atoms from a dictionary. Dictionary learning describes the acquisition process of the underlying dictionary for a given set of training samples.
While ideally this would be achieved by optimizing the expectation of the factors over the underlying distribution of the training data, in practice the necessary information about the distribution is not available. Therefore, in real world applications it is achieved by minimizing an empirical average over the available samples. 
The main goal of this paper is to provide a sample complexity estimate that controls to what extent the empirical average deviates from the cost function. This estimate then provides a suitable estimate to the accuracy of the representation of the learned dictionary. 
The presented approach exemplifies the general results proposed by the authors in \cite{gribonval:2014complexity} and gives more concrete bounds of the sample complexity of dictionary learning.
We cover a variety of sparsity measures employed in the learning procedure.

%The universal validity of the taken approach covers a variety of dictionary structures (e.g.\ 5unit norm columns, tensor product structure, sparse dictionaries, \ldots) as well as different sparsity inducing penalty functions. Thus, it provides a unified vantage point on the sample complexity of dictionary learning procedures.
%
\end{abstract}
\begin{keywords}
Dictionary learning, sample complexity, sparse coding
\end{keywords}
\section{Introduction}
\label{sec:intro}
The sparse synthesis model relies on the assumption that signals $\signal\in\R{m}$ can be represented as a sparse combination of columns, or atoms, of a dictionary $\Db \in \R{m \times d}$. As an equation this reads as
\begin{equation}
	\signal = \Db\coeff
\end{equation}
where $\coeff \in \R{d}$ is the sparse coefficient vector. 

The task of dictionary learning focuses on finding the best dictionary to sparsely represent a set of training samples concatenated in the matrix $\Signal=[\signal_1,\ldots,\signal_n]$. The corresponding sparse representations are stored in the coefficient matrix $\Coeff = [\coeff_1,\ldots,\coeff_n]$. A common learning approach is to find a solution to the minimization problem
\begin{align}
	\min_{\Db, \Coeff}\quad &\Lcal_\Signal(\Db,\Coeff)\label{eq:DL_min_prob}\\[-0.5em] 
	%\intertext{with}
	\Lcal_\Signal(\Db,\Coeff) \coloneqq \tfrac{1}{2n} \|&\Signal - \Db\Coeff\|_F^2 + \tfrac{1}{n}\sum_{i=1}^{n} g(\coeff_i).\label{eq:DL_cost_function}
\end{align}
The function $g$ in \eqref{eq:DL_cost_function} serves as a measure of sparsity. Concretely, we consider scaled powers of the $\ell_p$-norm, i.e.\ 
\begin{equation}
	%g(\coeff) \coloneqq \|\coeff\|_p^q
	g(\coeff) \coloneqq \|\coeff / \lambda\|_p^q
\end{equation}
for any $p,q>0$ and the weighting parameter $\lambda>0$.

In order to avoid trivial solutions, the learned dictionary $\Db$ is generally forced to be an element of a constraint set $\Dcal$. In this paper we will focus on dictionaries with unit $\ell_2$-norm atoms, which is a commonly used constraint.\\
The vast amount of dictionary learning algorithms that take different approaches to the topic illustrates the popularity of the synthesis model. A probabilistic method is presented in \cite{Olshausen:1997wz}. Another famous representative is the K-SVD algorithm as proposed in \cite{dl:aharon:2006} which is based on $K$-means clustering. Finally, there are learning strategies that aim at learning dictionaries with specific structures that enable fast computations, see e.g.\ \cite{dl:rubinstein:2010,cvpr:hawe:2013}.\\
Assuming that these training samples are drawn according to some distribution, the goal of dictionary learning is to find a dictionary $\Db^\star$ for which the expected value of the cost function \eqref{eq:DL_cost_function} is minimal.
In practice the distribution of the available training samples is unknown and therefore only an empirical minimizer $\hat{\Db}$ can be obtained.
The sample complexity results which we derive in this paper contribute to understand how accurately this empirical minimizer approximates $\Db^\star$.

We assume the training signals to be drawn according to a distribution in the ball with unit radius, i.e.\ the distribution is an element of the set
\begin{equation}
\label{eq:def_pfrak}
	\Pfrak \coloneqq \{ \PP \,:\, \PP(\|\signal\|_2 \leq 1) = 1 \}.
\end{equation}

%\subsection{Related Work}
Our results are based on the work \cite{gribonval:2014complexity} where a more general framework of matrix factorizations has been considered. Our stringent setting here allows for more concrete bounds on the sample complexity.

Previous state of the art sample complexity results are presented in \cite{MaurerPontil,vainsencher:2010}. These results are restricted to the indicator function for $\ell_0$ and $\ell_1$-norms. These works also cover the case of fast rates which we will not consider here.

\section{Problem Statement \& Notation}
\label{sec:problem}
Matrices are denoted by boldface capital letters such as $\Signal$, vectors will be represented as boldface small letters, e.g.\ $\signal$. Scalars will be slanted letters like $n,N$. The $i^\mathrm{th}$ entry of a vector $\coeff$ is denoted by $\alpha_i$. Finally, sets are denoted in blackletter such as $\Dcal$.

The sparse representation of a given signal $\signal$ according to a given dictionary can be found by solving the optimization problem
\begin{equation}
	\arg\!\min_{\coeff \in \R{d}} \tfrac{1}{2}\|\signal - \Db \coeff\|_2^2 + g(\coeff).
\end{equation}
The quality of how well a signal can be sparsely coded for a dictionary is evaluated via the function
\begin{equation}\begin{split}
	%\Lcal_\signal(\Db,\coeff) &\coloneqq \tfrac{1}{2}\|\signal - \Db \coeff\|_2^2 + g(\coeff),\\
	%f_\signal(\Db) &\coloneqq \inf_{\coeff \in \R{d}} \Lcal_\signal(\Db,\coeff).
	f_\signal(\Db) \coloneqq \inf_{\coeff \in \R{d}} \tfrac{1}{2}\|\signal - \Db \coeff\|_2^2 + g(\coeff).
\end{split}\end{equation}
For a set of signals $\Signal$ the overall quality of the sparse representation is measured via
\begin{equation}
	F_\Signal(\Db) \coloneqq \inf_\Coeff \Lcal_\Signal(\Db,\Coeff)
\end{equation}
with $\Lcal_\Signal$ as defined in $\eqref{eq:DL_cost_function}$. This measure is equal to the mean of the quality of all samples, i.e.\ $F_\Signal(\Db) = \frac{1}{n}\sum_{i} f_{\signal_i}(\Db)$.

%In order to keep the notation simple, we introduce the function
%%
%\begin{equation}
	%\Lcal_\signal(\Db,\coeff) \coloneqq \tfrac{1}{2}\|\signal - \Db \coeff\|_2^2 + g(\coeff).
%\end{equation}
%%
%The quality of how well a signal can be sparsely coded for a dictionary is evaluated via the function
%%
%\begin{equation}
	%f_\signal(\Db) \coloneqq \inf_{\coeff \in \R{d}} \Lcal(\Db,\coeff).
%\end{equation}
%
%We measure the overall quality of the sparse representation of a set of signals via
%%
%\begin{equation}
	%F_\Signal(\Db) \coloneqq \tfrac{1}{n} \sum_i f_{\signal_i}(\Db).
%\end{equation}
%%
%In the following we will use the short-hand notation $F_\Signal(\Db)=\inf_\Coeff \Lcal_\Signal(\Db,\Coeff)$ with $\Lcal_\Signal$ as defined in $\eqref{eq:DL_cost_function}$.

Our goal is to provide a bound for the generalization error of the empirical minimizer, i.e.\
\begin{equation}
	\sup_{\Db \in \Dcal} | F_\Signal(\Db) - \EE_{\signal \sim \PP} f_\signal(\Db)| \leq \eta(n,m,d,L)
\end{equation}
which depends on the number of samples $n$ used in the learning process, the size of the samples $m$, the number of dictionary atoms $d$, and a certain Lipschitz constant $L$ for $F_\Signal$.

%, where $g(\Coeff)$, by slight abuse of notation, denotes the mean of $g$ evaluated for all columns of $\Coeff$, i.e.\ $\frac{1}{n}\sum_i g(\coeff_i)$.

\section{Sample Complexity Results}
\label{sec:samp_comp}
The general strategy will be to first find a Lipschitz constant for the functions $F_\Signal$ and $\EE f_{\signal}$. In combination with the assumed underlying probability distribution and the structure of the dictionary this will allow us to provide an upper bound for the sample complexity.

\subsection{Lipschitz Property}
In this section we provide Lipschitz constants for the function $F_\Signal(\Db)$. 
%In order to perform the following calculations, we should take note of some required properties of $\ell_p^q$. It is a non-negative, lower semi-continuous, coercive (i.e.\ if $\|\coeff\|_1\to \infty$ then $\|\coeff\|_p^q\to \infty$) function with $\|\mathbf{0}\|_p^q = 0$.\textcolor{red}{unten einbauen}
For the $\ell_p$-penalty function the H\"older inequality yields
\begin{equation*}\begin{split}
	\|\coeff\|_1 = \sum_{i=1}^n |\alpha_i| &\leq \bigg(\sum_{i=1}^n |\alpha_i|^p \bigg)^{1/p} \bigg( \sum_{i=1}^n 1^{1/(1-1/p)}\bigg)^{1-1/p}\\
	{}&= d^{1-1/p}\cdot \|\coeff\|_p
\end{split}\end{equation*}
for $1\!\leq\! p\! <\! +\infty$. To cover quasi-norms with $0\!<\!p\!<\!1$, we employ the estimate $\|\coeff\|_1 \leq \|\coeff\|_p$ which provides the overall inequality
\begin{equation}
	\|\coeff\|_1 \leq d^{(1 - 1/p)_+}\cdot\|\coeff\|_p,
\end{equation}
where the function $(\cdot)_+\colon\R{}\!\to\!\R{+}_0$ is defined as $(t)_+ = \max\{t,0\}$.
Thus, we get the two estimates
\begin{align}
\|\coeff\|_p \leq t \quad &\Rightarrow\quad \|\coeff\|_1 \leq d^{(1 - 1/p)_+} \cdot t,\\ 
%\|\coeff\|_p^q \leq t \quad &\Rightarrow\quad \|\coeff\|_1 \leq d^{(1 - 1/p)_+} \cdot t^{1/q}\label{eq:gbar}
\|\coeff / \lambda\|_p^q \leq t \quad &\Rightarrow\quad \|\coeff\|_1 \leq \lambda \cdot d^{(1 - 1/p)_+} \cdot t^{1/q}\label{eq:gbar}
\end{align}
which will become of use in the following discussion. The matrix norm $\|\Coeff\|_{1 \to 2}\coloneqq \max_i \|\coeff_i\|_2$ is used for the rest of this paper while the subscript is omitted to improve readability. We also make use of the corresponding dual norm which is defined as $\|\Coeff\|_\star \coloneqq \sup_{\Ub,\|\Ub\|\leq 1} \langle \Coeff,\Ub\rangle_F$ for an appropriately sized matrix $\Ub$ and the Frobenius inner product $\langle \cdot,\cdot\rangle_F$.

For $\epsilon>0$ we define the nonempty set of matrices $\Coeff$ that are ``$\epsilon$-near solutions'' as
\begin{equation*}
	\Afrak_\epsilon(\Signal,\Db) \coloneqq \{ \Coeff : \coeff_i\in\R{d}, \Lcal_{\signal_i}(\Db,\coeff_i) \leq f_{\signal_i}(\Db) + \epsilon\}.
\end{equation*}
\begin{proposition}
The set $\Afrak_0$ is not empty and is bounded.
\end{proposition}
\begin{proof}
The function $\|\cdot\|_p^q$ is non-negative and coercive. Thus, $\Lcal_\Signal(\Db,\Coeff)$ is non-negative and $\lim_{k\to \infty} \Lcal_\Signal(\Db,\Coeff_k) = \infty$ whenever
$\lim_{k\to \infty}\|\Coeff_k\|=\infty$. Therefore, the function $\Coeff \mapsto \Lcal_{\Signal}(\Db,\Coeff)$ has bounded sublevel sets. Moreover, since powers of the $\ell_p$-norm are continuous, then so is $\Coeff \mapsto \Lcal_\Signal(\Db,\Coeff)$ and thus attains its infimum value. 
\end{proof}

Next, note that for any $\Dbp$ the inequality
\begin{equation}\begin{split}
\label{eq:local_lip_dependent}
	F_\Signal(\Dbp) &- F_\Signal(\Db)\\
	{}&\leq L_\Signal(\Db)\|\Dbp - \Db\| + C_\Signal(\Db)\|\Dbp - \Db\|^2
\end{split}\end{equation}
holds with
\begin{align}
	L_\Signal(\Db) &\coloneqq \inf_{\epsilon>0} \sup_{\Coeff \in \Afrak_\epsilon} \tfrac{1}{n} \|(\Signal - \Db\Coeff)\Coeff^\top\|_\star,\\[-0.6em] 
	C_\Signal(\Db) &\coloneqq \inf_{\epsilon>0} \sup_{\Coeff \in \Afrak_\epsilon} \tfrac{1}{2n} \sum_{i=1}^{n} \|\coeff_i\|_1^2.\label{eq:def_C_D}
\end{align}
A detailed derivation of these parameters can be found in \cite{gribonval:2014complexity}.

\begin{proposition}
\label{prop:L_independent}
	There exist upper bounds for $L_\Signal(\Db)$ and $C_\Signal(\Db)$ which are independent of the used dictionary.
\end{proposition}
\begin{proof}
For $\Coeff = [\coeff_1,\ldots,\coeff_n]\in\Afrak_0$  we have
\begin{equation*}
	\tfrac{1}{2}\|\signal_i - \Db \coeff_i\|_2^2  + \|\coeff_i / \lambda\|_p^q \leq f_{\signal_i}(\Db)
\end{equation*}
which immediately results in the estimates
\begin{align}
	0 \leq \|\coeff_i / \lambda\|_p^q \leq f_{\signal_i}&(\Db) \leq \Lcal(\Db,\mathbf{0}) = \tfrac{1}{2}\|\signal_i\|_2^2,\label{eq:apq_upper}\\
	\|\signal_i - \Db\coeff_i\|_2 &\leq \sqrt{2 f_{\signal_i}(\Db)} \leq \|\signal_i\|_2.\label{eq:accmeasure}
\end{align}
%

%By using the definition of $f_{\signal_i}$ we can further upper bound this via
%%
%\begin{align*}
	%0 \leq \|\coeff_i\|_p^q \leq f_{\signal_i}(\Db) &\leq \Lcal(\Db,\mathbf{0}) = \tfrac{1}{2}\|\signal_i\|_2^2,\\
	%\|\signal_i - \Db\coeff_i\|_2 &\leq \|\signal_i\|_2.
%\end{align*}
%%
Furthermore, the above in combination with Equation~\eqref{eq:gbar} lets us bound the $\ell_1$-norm of $\coeff$ via
\begin{equation}
	\|\coeff\|_1 \leq \lambda \cdot d^{(1-1/p)_+} \left( \tfrac{1}{2}\|\signal\|_2^2 \right)^{1/q}.
\end{equation}
This yields the upper bound $C_\Signal(\Db) \leq C_\Signal$ with
\begin{equation}
	C_\Signal \coloneqq \tfrac{1}{2n} \sum_{i=1}^{n} \lambda \cdot d^{(1-1/p)_+} \left( \tfrac{1}{2}\|\signal_i\|_2^2 \right)^{1/q}
\end{equation}
for Equation~\eqref{eq:def_C_D}.
In order to provide an upper bound for $L_\Signal(\Db)$ which is independent of the dictionary $\Db$ we first note that
\begin{align*}
	\langle (\Signal - \Db\Coeff)\Coeff^\top , \Ub \rangle \leq \sum \|\signal_i - \Db \coeff_i\|_2 \cdot \|\coeff_i\|_1 \cdot \|\Ub\|.
\end{align*}
By using the definition of the dual norm and the estimate developed above, we obtain the upper bound $L_\Signal(\Db) \leq L_\Signal$ with
\begin{equation}
\label{eq:def_LX}
	L_\Signal \coloneqq \tfrac{1}{n} \sum_{i=1}^n \|\signal_i\|_2  \cdot \lambda \cdot d^{(1-1/p)_+} \left( \tfrac{1}{2}\|\signal_i\|_2^2 \right)^{1/q}
\end{equation}
which concludes the proof.
\end{proof}

Proposition~\ref{prop:L_independent} allows us to rewrite Equation~\eqref{eq:local_lip_dependent} as
\begin{equation}
	\frac{|F_\Signal(\Dbp) - F_\Signal(\Db)|}{\|\Dbp - \Db\|} \leq L_\Signal \cdot \left(1 + \tfrac{C_\Signal}{L_\Signal}\|\Dbp - \Db\| \right),
\end{equation}
which implies that the function $F_\Signal$ is uniformly locally Lipschitz with constant $L_\Signal$. 
\begin{lemma}
The function $F_\Signal$ is globally Lipschitz with constant $L_\Signal$, i.e.
{\setlength\abovedisplayskip{5pt}
\setlength\belowdisplayskip{5pt}
\begin{equation}
\label{eq:lipschitz}
	|F_\Signal(\Dbp) - F_\Signal(\Db)| \leq L_\Signal \|\Dbp - \Db\|
\end{equation}}
for any $\Signal$ and any $\Db,\Dbp\in\Dcal$.
\end{lemma}
\begin{proof}
	Let $\epsilon>0$ be arbitrary but fixed. For $\|\Dbp - \Db\| \leq \epsilon L_\Signal / C_\Signal$ we have
	\begin{equation}
	\label{eq:glob_lip_proof}
		|F_\Signal(\Dbp) - F_\Signal(\Db)| \leq (1+\epsilon) L_\Signal \|\Dbp - \Db\|.
	\end{equation}
	If the bound on the distance between $\Db$ and $\Dbp$ does not hold, we construct the sequence $\Db_i \coloneqq \Db + \frac{i}{k}(\Dbp - \Db),\, i=0,\ldots,k$ such that $\|\Db_{i+1} - \Db_{i}\| \leq \epsilon L_\Signal / C_\Signal$. This enables us to show that the bound \eqref{eq:glob_lip_proof} holds for any $\Db,\Dbp$. Note that there are no restrictions on $\Db,\Dbp$, as the derived Lipschitz constant $L_\Signal$ is independent of the constraint set $\Dcal$.
	Since $\epsilon > 0$ is chosen arbitrarily, Equation~\eqref{eq:lipschitz} follows.
\end{proof}

\subsection{Probability Distribution}
As mentioned in the introduction we consider probability distributions within the unit ball. In order to obtain meaningful results the distribution according to which the samples are drawn has to fulfill two properties. First, we need to control the Lipschitz constant $L_\Signal$ for signals drawn according to the distribution when the number of samples $n$ is large. This quantity will be measured by the function
\begin{equation}
	\Lambda_n(L) \coloneqq \PP\left( L_\Signal > L\right).
\end{equation}
Furthermore, for a given $\Db$ we need to control the concentration of the empirical average $F_\Signal(\Db)$ around its expectation. This is measured via
\begin{equation}
\label{eq:concentration}
	%\Gamma_n(\gamma) \coloneqq \sup_{\Db \in \Dcal} \PP\left( \left|\tfrac{1}{n} \sum_i f_{\signal_i}(\Db) - \EE_{\signal\sim \PP}f_{\signal_i}\left(\Db\right) \right| > \gamma \right).
	\Gamma_n(\gamma) \coloneqq \sup_{\Db \in \Dcal} \PP\left( \left|F_\Signal(\Db) - \EE_{\signal\sim \PP}f_{\signal}\left(\Db\right) \right| > \gamma \right).
\end{equation}
For the considered distribution these quantities are well controlled, as can be seen in the following.
\begin{proposition}
	For $\PP \in \Pfrak$ as defined in Equation~\eqref{eq:def_pfrak} we have $\Lambda_n(L)=0$ for $L \geq d^{(1-1/p)_+}(1/2)^{1/q}$, and
	\begin{equation}
		\Gamma_n \left( \tau / \sqrt{8} \right) \leq 2 \exp(-n \tau^2),\ \forall\, 0\leq \tau < +\infty.
	\end{equation}
\end{proposition}

\begin{proof}
	The function evaluations $y_i = f_{\signal_i}(\Db)$ are independent random variables. For samples drawn according to a distribution within the unit sphere it immediately follows that they are bounded by $0 \leq y_i \leq \|\signal_i\|_2^2 / 2 \leq 1/2$. Using Hoeffding's Inequality we get
	\begin{align*}
		\PP\left[ \tfrac{1}{n}\left(\sum_{i=1}^n y_i - \EE[y_i]\right) \geq c\tau \right] \leq \exp\left( -8c^2 n \tau^2 \right)
	\end{align*}
	and thus $\Gamma_n(\tau/\sqrt{8}) \leq 2\exp(-n\tau^2)$.
	Furthermore, due to the chosen set of viable probability distributions $\Pfrak$, we have $L_\Signal \leq \lambda \cdot d^{(1-1/p)_+} (1/2)^{1/q}$ and hence $\Lambda_n(L) = 0$ for $L \geq \lambda \cdot d^{(1-1/p)_+} (1/2)^{1/q}$.
\end{proof}

\subsection{Role of the Constraint Set}
In order to provide a sample complexity bound, it is necessary to take the structure of the set to which the learned dictionary is constrained into account. Of particular interest is the covering number of the concerning set. For more information on covering numbers, the interested reader is referred to \cite{anthony2009neural}. %In combination with the Lipschitz property of the function $F_\Signal$ this will allow us to estimate the sample complexity. 

We will confine the discussion to the set of dictionaries with unit norm atoms, i.e.\ each dictionary column is an element of the sphere. It is well known that the covering number of the sphere is upper bounded by $\Ncal_\epsilon(\Sph^{m-1}) \leq \left( 1 + \tfrac{2}{\epsilon} \right)^m$. By using the metric $\|\cdot\|_{1 \to 2}$ this is readily extended to the product of unit spheres
\begin{equation}
	\Ncal_\epsilon\left(\Dcal(m,d)\right) \leq \left(1+ \tfrac{2}{\epsilon}\right)^{md} \leq \left(\tfrac{3}{\epsilon} \right)^{md}.
\end{equation}
%

%Another constraint set we want to consider are dictionaries with a separable structure as introduced in \cite{cvpr:hawe:2013}. These dictionaries can be expressed as the Kronecker product of smaller dictionaries $\Db = \bigotimes_{i=1}^{N} \Db_i$, $\Db_i \in \Dcal(m_i,d_i)$ for $i=1,\ldots,N$, and offer a significant improvement in terms of calculation cost both in the learning as well as in the reconstruction phase for multidimensional signals. The Covering number of a dictionary with separable structure is upper bounded by
%%
%\begin{equation}
	%\Ncal(\Dcal_\otimes,\epsilon) \leq \left( \frac{3}{\epsilon}\right)^{\sum_i m_i d_i}.
%\end{equation}
%%
%
%The authors of \cite{dl:rubinstein:2010} propose another dictionary structure, the so-called sparse dictionaries. Their assumption is that a dictionary itself can be sparsely represented via a fundamental dictionary. To obtain a viable interpretation for our framework this can be translated to a dictionary that in addition to unit norm atoms, each atom itself is also sparse. The covering numbers of the constraint set that enforces this property is upper bounded by
%%
%\begin{equation}
	%\Ncal(\Dcal_{s\text{ - }\mathrm{sparse}},\epsilon) \leq \left( \binom{m}{s} \left(\frac{3}{\epsilon} \right)^s \right)^d
%\end{equation}

\subsection{Main Result}
% ALLGEMEIN
%\begin{theorem}
%\label{thm:main}
%Given the Lipschitz constant $L > L_\Signal$ we define the constants
%\begin{align}
	%\beta &\coloneqq \tfrac{h}{8} \cdot \max \left\{ \log\left( 2\sqrt{8}LC \right), 1 \right\},\\
	%\eta_n(g,\Dcal,\Pfrak_{B_1}) &\coloneqq 2\sqrt{\frac{\beta \log n}{n}} + \sqrt{\frac{\beta + x/\sqrt{8}}{n}}.
%\end{align}
%%
%For a given $0 \leq x \leq \infty$, we have
%%
%\begin{equation}
	%\sup_{\Db \in \Dcal} \left| F_\Signal(\Db) - \EE_{\signal \sim \PP} f_{\signal}(\Db)\right| \leq \eta_n(g,\Dcal,\Pfrak_{B_1})
%\end{equation}
%%
%with probability at least $1 - 2e^{-x}$.
%\end{theorem}
%
%The parameters $C,h \geq 1$ depend on the constraint set for the dictionary, whereas $L$ is influenced by the choice of sparsifying function $g$. While we assume the distribution of the training signals to be fixed, we still included it as a dependency for $\eta_n$ in order to illustrate its influence on the sample complexity.
\begin{theorem}
\label{thm:main}
For a given $0\! \leq\! \vardev\! <\! \infty$ and the Lipschitz constant %$L > L_\Signal$, we have
$L > \lambda \cdot d^{(1-1/p)_+} (1/2)^{1/q}$, we have
%
%\begin{equation}
	%\sup_{\Db \in \Dcal} \left| F_\Signal(\Db) - \EE_{\signal \sim \PP} f_{\signal}(\Db)\right| \leq \eta_n(g,\Dcal,\Pfrak_{B})
%\end{equation}
\begin{equation}
	\sup_{\Db \in \Dcal} \left| F_\Signal(\Db) - \EE_{\signal \sim \PP} f_{\signal}(\Db)\right| \leq \eta(n,m,d,L)
\end{equation}
with probability at least $1 - 2e^{-\vardev}$. The bound is defined as
\begin{equation}
	\eta(n,m,d,L) \coloneqq 2\sqrt{\frac{\beta \log n}{n}} + \sqrt{\frac{\beta + \vardev/\sqrt{8}}{n}}
\end{equation}
with the driving constant
\begin{equation}
	\beta \coloneqq \tfrac{md}{8} \cdot \max \left\{ \log\left( 6\sqrt{8}L \right), 1 \right\}.
	%\beta \coloneqq \tfrac{md}{8} \cdot \max \left\{ \log\left( 6\sqrt{8}(1/2)^{1/q} \cdot \lambda \cdot d^{(1-1/p)_+} \right), 1 \right\}.
\end{equation}
\end{theorem}

The parameter controlling the sample complexity is dependent on the size of the dictionary, the determined Lipschitz constant, and the number of samples. It is also dependent on the underlying distribution of the samples, which is an arbitrary distribution in the unit ball in the examined case. Better estimates may hold for fixed probability distributions.

\begin{proof}
	%%%%%%%%%%%%%%%%%%%%%%%%%%%%%%%%%%%%%%%%%%
	%long version
	%%%%%%%%%%%%%%%%%%%%%%%%%%%%%%%%%%%%%%%%%%
	%First, we show that 
	%%
	%\begin{equation}
		%\sup_\Db |F_\Signal(\Db) - \EE f_\signal(\Db) | \leq 2L\epsilon + \gamma
	%\end{equation}
	%with probability at most $\Ncal(\Dcal,\epsilon)\cdot \Gamma_n(\gamma)$.
	%
	%For a fixed $\epsilon$ we can obtain a cover of the constraint set $\Dcal$ with at most $\Ncal(\Dcal,\epsilon)$ elements $\{\Db_i\}$. According to the nature of an $\epsilon$-covering there exits a $\Db_j \in \{\Db_i\}$ such that $\|\Db - \Db_j\|_{1 \to 2} \leq \epsilon$. Thus, we can provide the upper estimate
	%%
	%\begin{align*}
		%&\|F_\Signal(\Db) - \EE f_\signal(\Db)\| \\
			%&\leq |F_\Signal(\Db) - F_\Signal(\Db_j)| + |F_\Signal(\Db_j) - \EE f_\signal(\Db_j)| + |\EE f_\signal(\Db_j) - \EE f_\signal(\Db)| \\
			%&\leq |F_\Signal(\Db) - F_\Signal(\Db_j)| + sup_j |F_\Signal(\Db_j) - \EE f_\signal(\Db_j)| + L\epsilon.
	%\end{align*}
	%%
	%and finally we can upper bound $|F_\Signal(\Db) - F_\Signal(\Db_j)|$ by $L\epsilon$ and $sup_j |F_\Signal(\Db_j) - \EE f_\signal(\Db_j)|$ by $\gamma$ with probability $1-\Gamma_n(\gamma)$.
%	
	%%%%%%%%%%%%%%%%%%%%%%%%%%%%%%%%%%%%%%%%%%
	%short version
	%%%%%%%%%%%%%%%%%%%%%%%%%%%%%%%%%%%%%%%%%%
	First, note that $\EE f_\signal$ is Lipschitz with constant $L>L_\Signal$ for the considered case. Let $\epsilon,\gamma > 0$ be given.
	The constraint set $\Dcal$ can be covered by and $\epsilon$-net with $\Ncal_\epsilon$ elements $\Db_j$. For a fixed dictionary $\Db$ there exists an index $j$ for which $\|\Db - \Db_j\| \leq \epsilon$ and we can write
	\begin{equation*}\begin{split}
		|F_\Signal(\Db) - \EE f_\signal(\Db) | \leq& |F_\Signal(\Db) - F_\Signal(\Db_j)|\\
		{}&+ \sup_j |F_\Signal(\Db_j) - \EE f_\signal(\Db_j)| \\
		{}&+ |\EE f_\signal(\Db_j) - \EE f_\signal(\Db)|.
	\end{split}\end{equation*}
	By using the concentration assumption \eqref{eq:concentration} and the Lipschitz continuity of $F_\Signal$ and $\EE f_\signal$ this implies
	\begin{equation}
		\sup_\Db |F_\Signal(\Db) - \EE_{\signal \sim \PP} f_\signal(\Db) | \leq 2L\epsilon + \gamma
	\end{equation}
	except for probability at most $\Ncal_\epsilon\cdot \Gamma_n(\gamma)$.
	Since the above holds for and $\epsilon,\gamma>0$, we specify the constants
	\begin{align*}
		\epsilon &\coloneqq \tfrac{1}{2L}\sqrt{\tfrac{\beta\log n}{n}},\\
		\tau &\coloneqq \tfrac{1}{\sqrt{n}} \sqrt{md \log\left( \tfrac{3}{\epsilon} \right) + \vardev}
	\end{align*}
	with $\gamma \coloneqq \tau/\sqrt{8}$ which fulfill the conditions $0\!<\!\epsilon\!<\!1$ and $0\!\leq\! \tau \!<\! \infty$. Given these parameters we get
	\begin{equation}
		\Ncal_\epsilon \cdot \Gamma_n(\tau/\sqrt{8}) = 2 e^{-\vardev}.
	\end{equation}
	For the final estimate, recall that due to the definition of $\beta$ the inequalities
	\begin{equation*}
		\log\left( \tfrac{6L}{\sqrt{\beta}}\right) \leq \log\left( 6\sqrt{8}L \right) \leq 8\beta/(md)
	\end{equation*}
	and $\log n \geq 1$ hold. This allows us to provide the estimate
	%
	%\begin{equation}\begin{split}
		%2L\epsilon &+ \tau/\sqrt{8} \\
			%&= \tfrac{1}{\sqrt{8}} \sqrt{\beta \log n / n} + \tfrac{1}{\sqrt{8n}} \sqrt{h \log\left( \tfrac{C}{\epsilon} \right) + \vardev}\\
			%&= \tfrac{1}{\sqrt{8}} \sqrt{\tfrac{\beta \log n}{n}} + \tfrac{1}{\sqrt{8n}} \sqrt{h \log\left( \tfrac{2\sqrt{8}LC}{\sqrt{\beta}} \right) + \tfrac{h}{2}\log\left( \tfrac{n}{\log n} \right) + \vardev}\\
			%&\leq \tfrac{1}{\sqrt{8}} \sqrt{\beta \log n / n} + \tfrac{1}{\sqrt{8n}} \sqrt{\beta + \tfrac{\beta}{2} \log n +\vardev}\\
			%&\leq \tfrac{1}{\sqrt{8}} \sqrt{\beta \log n / n} \cdot (1 + 1/\sqrt{2}) + \tfrac{1}{\sqrt{8}} \sqrt{ (\beta + \vardev)/n}\\
			%&\leq \tfrac{2}{\sqrt{8}} \sqrt{\beta \log n / n} + \tfrac{1}{\sqrt{8}} \sqrt{ (\beta + \vardev)/n}.
	%\end{split}\end{equation}
	%
	\begin{equation}\begin{split}
		2L\epsilon + \tau/\sqrt{8}
			%&= \sqrt{\beta \log n / n} + \tfrac{1}{\sqrt{8n}} \sqrt{h \log\left( \tfrac{C}{\epsilon} \right) + \vardev}\\
			%&= \sqrt{\tfrac{\beta \log n}{n}} + \tfrac{1}{\sqrt{8n}} \sqrt{h \log\left( \tfrac{2LC}{\sqrt{\beta}} \right) + \tfrac{h}{2}\log\left( \tfrac{n}{\log n} \right) + \vardev}\\
			%&\leq \sqrt{\beta \log n / n} + \tfrac{1}{\sqrt{n}} \sqrt{\beta + \tfrac{\beta}{2} \log n +\vardev}\\
			%&\leq \sqrt{\beta \log n / n} \cdot (1 + 1/\sqrt{2}) + \sqrt{ (\beta + \vardev)/n}\\
			\leq 2\sqrt{\tfrac{\beta \log n}{n}} + \sqrt{ \tfrac{\beta + \vardev/\sqrt{8}}{n}}
	\end{split}\end{equation}
	which concludes the proof of Theorem~\ref{thm:main}.
\end{proof}

In order to illustrate the results, we will discuss a short example.
\begin{example}
	The general assumption is that the learned dictionary is an element of $\Dcal(m,d)$ and the training samples are drawn according to a distribution in the unit ball. Let the sparsity promoting function be defined as the $\ell_p$-norm with $0\!<\!p\!<\!1$. Then Theorem~\ref{thm:main} holds with the sample complexity driving constant %$\beta = \tfrac{md}{8} \cdot \log(3 \sqrt{8})$.
	\vspace{-2mm}\begin{equation}
		\beta = \tfrac{md}{8} \cdot \log(3 \sqrt{8}).
	\end{equation}
\end{example}

\vspace{-2.1mm}

\section{Conclusion}
\label{sec:conc}
We provide a sample complexity result for learning dictionaries with unit norm atoms. Powers of the $\ell_p$-norm as a penalty in the learning process while the samples are drawn according to a distribution within the unit ball. In general, we can say that the sample complexity bound $\eta$ exhibits the behavior $\eta \propto \sqrt{\frac{\log n}{n}}$ with high probability. The sample complexity results achieved in this paper recover those of \cite{MaurerPontil,vainsencher:2010} for a different choice of sparsity measure. The more general framework behind this work is presented in \cite{gribonval:2014complexity} and is exploited here to provide more concrete results specified for a particular dictionary learning case.

% References should be produced using the bibtex program from suitable
% BiBTeX files (here: strings, refs, manuals). The IEEEbib.bst bibliography
% style file from IEEE produces unsorted bibliography list.
% -------------------------------------------------------------------------
\bibliographystyle{IEEEbib}
\bibliography{refs}

\end{document}